\def\year{2018}\relax
\documentclass[letterpaper]{article} 
\usepackage{aaai18}  
\usepackage{times}  
\usepackage{helvet}  
\usepackage{courier}  
\usepackage{url}  
\usepackage{graphicx}  
\usepackage{amsmath,amscd,amsbsy,amssymb,amsthm}
\usepackage{algorithm}
\usepackage{algorithmic}
\usepackage{appendix}
\usepackage{subfigure}
\newtheorem{theorem}{Theorem}
\newtheorem{lemma}{Lemma}

\begin{document}
%
\title{Wasserstein Distance Guided Representation Learning\\for Domain Adaptation}
\author{Jian Shen, Yanru Qu, Weinan Zhang\thanks{Weinan Zhang is the corresponding author.}, Yong Yu\\
Shanghai Jiao Tong University\\
\{rockyshen, kevinqu, wnzhang, yyu\}@apex.sjtu.edu.cn
}
\maketitle
\begin{abstract}
Domain adaptation aims at generalizing a high-performance learner on a target domain via utilizing the knowledge distilled from a source domain which has a different but related data distribution. One solution to domain adaptation is to learn domain invariant feature representations while the learned representations should also be discriminative in prediction. To learn such representations, domain adaptation frameworks usually include a domain invariant representation learning approach to measure and reduce the domain discrepancy, as well as a discriminator for classification. Inspired by Wasserstein GAN, in this paper we propose a novel approach to learn domain invariant feature representations, namely Wasserstein Distance Guided Representation Learning (WDGRL). WDGRL utilizes a neural network, denoted by the domain critic, to estimate empirical Wasserstein distance between the source and target samples and optimizes the feature extractor network to minimize the estimated Wasserstein distance in an adversarial manner. The theoretical advantages of Wasserstein distance for domain adaptation lie in its gradient property and promising generalization bound. Empirical studies on common sentiment and image classification adaptation datasets demonstrate that our proposed WDGRL outperforms the state-of-the-art domain invariant representation learning approaches.
\end{abstract}

\section{Introduction}

Domain adaptation defines the problem when the target domain labeled data is insufficient, while the source domain has much more labeled data. Even though the source and target domains have different marginal distributions \cite{ben2007analysis,pan2010survey}, domain adaptation aims at utilizing the knowledge distilled from the source domain to help target domain learning.
In practice, unsupervised domain adaptation is concerned and studied more commonly since manual annotation is often expensive or time-consuming. Faced with the covariate shift and the lack of annotations, conventional machine learning methods may fail to learn a high-performance model.

To effectively transfer a classifier across different domains, different methods have been proposed, including instance reweighting \cite{mansour2009domain}, subsampling \cite{chen2011automatic}, feature mapping \cite{tzeng2014deep} and weight regularization \cite{rozantsev2016beyond}. Among these methods feature mapping has shown great success recently, which projects the data from different domains to a common latent space where the feature representations are domain invariant. Recently, deep neural networks, as a great tool to automatically learn effective data representations, have been leveraged in learning knowledge-transferable feature representations for domain adaptation \cite{glorot2011domain,chen2012marginalized,zhuang2015supervised,long2015learning,ganin2016domain}.

On the other hand, generative adversarial nets (GANs) \cite{goodfellow2014generative} are heavily studied during recent years, which play a minimax game between two adversarial networks: the discriminator is trained to distinguish real data from the generated data, while the generator learns to generate high-quality data to fool the discriminator. It is intuitive to employ this minimax game for domain adaptation to make the source and target feature representations indistinguishable. These adversarial adaptation methods have become a popular solution to reduce domain discrepancy through an adversarial objective with respect to a domain classifier \cite{ganin2016domain,tzeng2017adversarial}. However, when the domain classifier network can perfectly distinguish target representations from source ones, there will be a gradient vanishing problem. A more reasonable solution would be to replace the domain discrepancy measure with Wasserstein distance, which provides more stable gradients even if two distributions are distant \cite{arjovsky2017wasserstein}.

In this paper, we propose a domain invariant representation learning approach to reduce domain discrepancy for domain adaptation, namely Wasserstein Distance Guided Representation Learning (WDGRL), inspired by recently proposed Wasserstein GAN \cite{arjovsky2017wasserstein}. WDGRL trains a domain critic network to estimate the empirical Wasserstein distance between the source and target feature representations. The feature extractor network will then be optimized to minimize the estimated Wasserstein distance in an adversarial manner. By iterative adversarial training, we finally learn feature representations invariant to the covariate shift between domains. Additionally, WDGRL can be easily adopted in existing domain adaptation frameworks \cite{tzeng2014deep,long2015learning,zhuang2015supervised,long2016deep,bousmalis2016domain} by replacing the representation learning approaches. Empirical studies on common domain adaptation benchmarks demonstrate that WDGRL outperforms the state-of-the-art representation learning approaches for domain adaptation. Furthermore, the visualization of learned representations clearly shows that WDGRL successfully unifies two domain distributions, as well as maintains obvious label discrimination.

\section{Related Works}

Domain adaptation is a popular subject in transfer learning \cite{pan2010survey}. It concerns covariate shift between two data distributions, usually labeled source data and unlabeled target data. Solutions to domain adaptation problems can be mainly categorized into three types: i). Instance-based methods, which reweight/subsample the source samples to match the distribution of the target domain, thus training on the reweighted source samples guarantees classifiers with transferability \cite{huang2007correcting,chen2011co,chu2013selective}. ii). Parameter-based methods, which transfer knowledge through shared or regularized parameters of source and target domain learners, or by combining multiple reweighted source learners to form an improved target learner \cite{duan2012exploiting,rozantsev2016beyond}. iii). The last but the most popular and effective methods are feature-based, which can be further categorized into two groups \cite{weiss2016survey}. Asymmetric feature-based methods transform the features of one domain to more closely match another domain \cite{hoffman2014asymmetric,kandemir2015asymmetric,courty2017optimal} while symmetric feature-based methods map different domains to a common latent space where the feature distributions are close.

Recently, deep learning has been regarded as a powerful way to learn feature representations for domain adaptation. Symmetric feature-based methods are more widely studied since it can be easily incorporated into deep neural networks \cite{chen2012marginalized,zhuang2015supervised,long2015learning,ganin2016domain,bousmalis2016domain,luo2017close}.
Among symmetric feature-based methods, minimizing the maximum mean discrepancy (MMD) \cite{gretton2012kernel} metric is effective to minimize the divergence of two distributions. MMD  is a nonparametric metric that measures the distribution divergence between the mean embeddings of two distributions in reproducing kernel Hilbert space (RKHS). The deep domain confusion (DDC) method \cite{tzeng2014deep} utilized MMD metric in the last fully connected layer in addition to the regular classification loss to learn representations that are both domain invariant and discriminative. Deep adaptation network (DAN) \cite{long2015learning} was proposed to enhance the feature transferability by minimizing multi-kernel MMD in several task-specific layers. On the other hand, correlation alignment (CORAL) method \cite{sun2016return} was proposed to align the second-order statistics of the source and target distributions with a linear transformation and \cite{sun2016deep} extended CORAL and proposed Deep CORAL to learn a nonlinear transformation that aligns correlations of layer activations in deep neural networks. 

Another class of symmetric feature-based methods uses an adversarial objective to reduce domain discrepancy. Motivated by theory in \cite{ben2007analysis,ben2010theory} suggesting that a good cross-domain representation contains no discriminative information about the origin (i.e. domain) of the input, domain adversarial neural network (DANN) \cite{ajakan2014domain,ganin2016domain} was proposed to learn domain invariant features by a minimax game between the domain classifier and the feature extractor. In order to back-propagate the gradients computed from the domain classifier, DANN employs a gradient reversal layer (GRL). On the other hand, \cite{tzeng2017adversarial} proposed a general framework for adversarial adaptation by choosing adversarial loss type with respect to the domain classifier and the weight sharing strategy. Our proposed WDGRL can also be viewed as an adversarial adaptation method since it evaluates and minimizes the empirical Wasserstein distance in an adversarial manner. Our WDGRL differs from previous adversarial methods: i). WDGRL adopts an iterative adversarial training strategy, ii). WDGRL adopts Wasserstein distance as the adversarial loss which has gradient superiority.

Another related work for domain adaptation is optimal transport \cite{courty2014domain,courty2017optimal}, which is equivalent to Wasserstein distance. And \cite{redko2016theoretical} gave a theoretical analysis that Wasserstein distance can guarantee generalization for domain adaptation. Though these works utilized Wasserstein distance in domain adaptation, there are distinct differences between WDGRL and the previous ones: these works are asymmetric feature-based methods which design a transformation from source representations to target ones based on optimal transport while WDGRL is a symmetric method that projects both domains to a common latent space to learn domain invariant features. And WDGRL can be integrated into other symmetric feature-based adaptation frameworks. 

Besides learning shared representations, domain separation network (DSN)  \cite{bousmalis2016domain} was proposed to explicitly separate private representations for each domain and shared ones between the source and target domains. The private representations were learned by defining a difference loss via a soft orthogonality constraint between the shared and private representations while the shared representations were learned by DANN or MMD mentioned above. With the help of reconstruction through private and shared representations together, the classifier trained on the shared representations can better generalize across domains. Since our work focuses on learning the shared representations, it can also be integrated into DSN easily.

\section{Wasserstein Metric}

Before we introduce our domain invariant feature representation learning approach, we first give a brief introduction of the Wasserstein metric. The Wasserstein metric is a distance measure between probability distributions on a given metric space $(M, \rho)$, where $\rho(x,y)$ is a distance function for two instances $x$ and $y$ in the set $M$. The $p{\text{-th}}$ Wasserstein distance between two Borel probability measures $\mathbb{P} $ and $\mathbb{Q}$ is defined as
\begin{equation}
W_p(\mathbb{P}, \mathbb{Q}) = \Big(\inf_{\mu \in \Gamma(\mathbb{P}, \mathbb{Q}) } \int \rho(x,y)^p d\mu(x,y) \Big)^{1/p},
\end{equation}
where $\mathbb{P}, \mathbb{Q} \in \{\mathbb{P} : \int \rho (x,y) ^p d\mathbb{P}(x) < \infty , \forall y \in M \} $ are two probability measures on $M$ with finite $p{\text{-th}}$ moment and $\Gamma(\mathbb{P}, \mathbb{Q})$ is the set of all measures on $M \times M$ with marginals $\mathbb{P}$ and $\mathbb{Q}$. Wasserstein metric arises in the problem of optimal transport: $\mu(x,y)$ can be viewed as a randomized policy for transporting a unit quantity of some material from a random location $x$ to another location $y$ while satisfying the marginal constraint $x \sim \mathbb{P}$ and $y \sim \mathbb{Q}$. If the cost of transporting a unit of material from $x \in \mathbb{P}$ to $y \in \mathbb{Q}$ is given by $\rho(x,y)^p$, then $W_p(\mathbb{P}, \mathbb{Q})$ is the minimum expected transport cost.

The Kantorovich-Rubinstein theorem shows that when $M$ is separable, the dual representation of the first Wasserstein distance (Earth-Mover distance) can be written as a form of integral probability metric \cite{villani2008optimal}
\begin{equation}
W_1(\mathbb{P},\mathbb{Q})= \sup_{\left \| f \right \|_L \leq 1} \mathbb{E}_{x \sim \mathbb{P}}[f(x)] - \mathbb{E}_{x \sim \mathbb{Q}}[f(x)], \label{eq:w1-distance}
\end{equation}
where the Lipschitz semi-norm is defined as $\left \| f \right \|_L = \sup{|f(x) - f(y)|} / \rho(x,y)$. In this paper, for simplicity, Wasserstein distance represents the first Wasserstein distance.

\section{Wasserstein Distance Guided \\ Reprensentation Learning}

\subsection{Problem Definition}

In unsupervised domain adaptation problem, we have a labeled source dataset $X^s=\{(x^s_i,y^s_i)\}_{i=1}^{n^s}$ of $n^s$ samples from the source domain $\mathcal{D}_s$ which is assumed sufficient to train an accurate classifier, and an unlabeled target dataset $X^t = \{x_j^t\}^{n^t}_{j=1}$ of $n^t$ samples from the target domain $\mathcal{D}_t$. It is assumed that the two domains share the same feature space but follow different marginal data distributions, $\mathbb{P}_{x^s}$ and $\mathbb{P}_{x^t}$ respectively. The goal is to learn a transferable classifier $\eta(x)$ to minimize target risk $\epsilon_{t} = \mathrm{Pr}_{(x,y) \sim \mathcal{D}_t} [ \eta(x) \neq y ]$ using all the given data.

\subsection{Domain Invariant Representation Learning}

The challenge of unsupervised domain adaptation mainly lies in the fact that two domains have different data distributions. Thus the model trained with source domain data may be highly biased in the target domain. To solve this problem, we propose a new approach to learn feature representations invariant to the change of domains by minimizing empirical Wasserstein distance between the source and target representations through adversarial training.

In our adversarial representation learning approach, there is a feature extractor which can be implemented by a neural network. The feature extractor is supposed to learn the domain invariant feature representations from both domains. Given an instance $x \in \mathbb{R}^{m}$ from either domain,  the feature extractor learns a function $f_g: \mathbb{R}^m \rightarrow \mathbb{R}^d$ that maps the instance to a $d$-dimensional representation with corresponding network parameter $\theta_g$. And then in order to reduce the discrepancy between the source and target domains, we use the domain critic, as suggested in \cite{arjovsky2017wasserstein}, whose goal is to estimate the Wasserstein distance between the source and target representation distributions. Given a feature representation $h = f_g(x)$ computed by the feature extractor, the domain critic learns a function $f_w:\mathbb{R}^d \rightarrow \mathbb{R}$ that maps the feature representation to a real number with parameter $\theta_w$. Then the Wasserstein distance between two representation distributions $\mathbb{P}_{h^s}$ and $\mathbb{P}_{h^t}$, where $h^s = f_g(x^s)$ and $h^t=f_g(x^t)$, can be computed according to Eq.~(\ref{eq:w1-distance})
{\small
\begin{equation}
\begin{aligned}
W_1(\mathbb{P}_{h^s},\mathbb{P}_{h^t}) & = \sup_{\left \| f_w \right \|_L \leq 1} \mathbb{E}_{\mathbb{P}_{h^s}} [f_w(h)] - \mathbb{E}_{\mathbb{P}_{h^t}  }[f_w(h)] \\
	&=  \sup_{\left \| f_w \right \|_L \leq 1} \mathbb{E}_{\mathbb{P}_{x^s}} [f_w(f_g(x))] - \mathbb{E}_{\mathbb{P}_{x^t}  }[f_w(f_g(x))].  \label{wd-for-da}
\end{aligned}   
\end{equation}          
}
If the parameterized family of domain critic functions $\{f_w\}$ are all $1$-Lipschitz, then we can approximate the empirical Wasserstein distance by maximizing the domain critic loss $\mathcal{L}_{wd} $ with respect to parameter $\theta_w$
\begin{equation}
\mathcal{L}_{wd}(x^s, x^t) \! = \! \frac{1}{n^s} \! \sum_{x^s \in X^s} \! f_w(f_g(x^s)) \!-\! \frac{1}{n^t} \! \sum_{x^t \in X^t} \! f_w(f_g(x^t)).
\end{equation}
Here comes the question of enforcing the Lipschitz constraint. \cite{arjovsky2017wasserstein} proposed to clip the weights of domain critic within a compact space $[-c, c]$ after each gradient update. However \cite{gulrajani2017improved} pointed out that weight clipping will cause capacity underuse and gradient vanishing or exploding problems. As suggested in \cite{gulrajani2017improved}, a more reasonable way is to enforce gradient penalty $\mathcal{L}_{grad}$ for the domain critic parameter $\theta_w$
\begin{equation}
\mathcal{L}_{grad}(\hat{h}) =(\| \nabla_{\hat{h}} f_w(\hat{h}) \|_2-1)^2,
\end{equation}
where the feature representations $\hat{h} $ at which to penalize the gradients are defined not only at the source and target representations but also at the random points along the straight line between source and target representation pairs. So we can finally estimate the empirical Wasserstein distance by solving the problem
\begin{equation}
\max_{\theta_w} \{ \mathcal{L}_{wd} - \gamma \mathcal{L}_{grad}\}
\end{equation}
where $\gamma$ is the balancing coefficient.

Since the Wasserstein distance is continuous and differentiable almost everywhere, we can first train the domain critic to optimality. Then by fixing the optimal parameter of domain critic and minimizing the estimator of Wasserstein distance, the feature extractor network can learn feature representations with domain discrepancy reduced. Up to now the representation learning can be achieved by solving the minimax problem
\begin{equation}
\min_{\theta_g}\max_{\theta _w} \{\mathcal{L}_{wd} - \gamma \mathcal{L}_{grad} \}
\end{equation}
where $\gamma$ should be set $0$ when optimizing the minimum operation since the gradient penalty should not guide the representation learning process. By iteratively learning feature representations with lower Wasserstein distance, the adversarial objective can finally learn domain invariant feature representations.

\subsection{Combining with Discriminator}

As mentioned above, our final goal is to learn a high-performance classifier for the target domain. However, the process of WDGRL is in an unsupervised setting, which may result in that the learned domain invariant representations are not discriminative enough. 
Hence it is necessary to incorporate the supervision signals of source domain data into the representation learning process as in DANN \cite{ganin2016domain}. Next we further introduce the combination of the representation learning approaches and a discriminator, of which the overview framework is given by Figure~\ref{fig:wd-tl}.  A detailed algorithm of the combination is given in Algorithm~\ref{alg:framework}.

\begin{figure}[tpb]
	\centering 
	\includegraphics[width=0.45\textwidth]{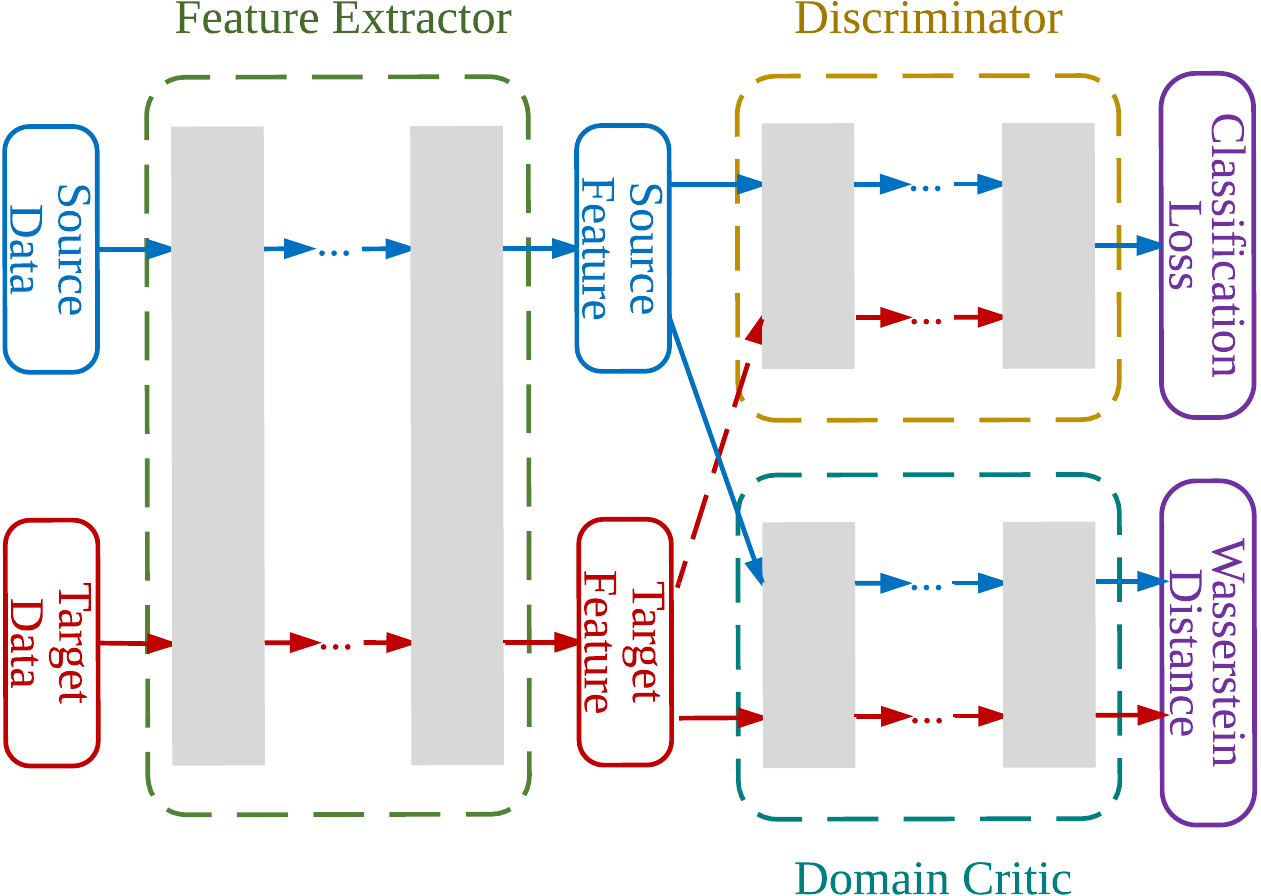}
	\caption{WDGRL Combining with Discriminator.}
	\label{fig:wd-tl}
\end{figure}

We further add several layers as the discriminator after the feature extractor network. Since WDGRL guarantees transferability of the learned representations, the shared discriminator can be directly applied to target domain prediction when training finished. 
The objective of the discriminator $f_c : \mathbb{R}^d \rightarrow \mathbb{R}^ l $ is to compute the softmax prediction with parameter $\theta_c $ where $l$ is the number of classes. The  discriminator loss function is defined as the cross-entropy between the predicted probabilistic distribution and the one-hot encoding of the class labels given the labeled source data:
\begin{equation}
\mathcal{L}_c(x^s, y^s) = -\frac{1}{n^s} \sum_{i=1}^{n^s} \sum_{k=1}^l 1(y^s_i = k) \cdot \log {f_c(f_g(x_i^s))}_k ,
\end{equation}
where $1(y^s_i = k)$ is the indicator function and ${f_c(f_g(x_i^s))}_k$ corresponds to the $k$-th dimension value of the distribution $f_c(f_g(x_i^s))$.
By combining the discriminator loss, we attain our final objective function
\begin{equation}
\min_{ \theta_g, \theta_c}  \Big\{ \mathcal{L}_c  + \lambda \max_{\theta _w}\Big[ \mathcal{L}_{wd} - \gamma \mathcal{L}_{grad} \Big] \Big\},
\end{equation}
where $\lambda $ is the coefficient that controls the balance between discriminative and transferable feature learning and $\gamma$ should be set 0 when optimizing the minimum operator.

Note that this algorithm can be trained by the standard back-propagation with two iterative steps. In a mini-batch containing labeled source data and unlabeled target data, we first train the domain critic network to optimality by optimizing the max operator via gradient ascent and then update the feature extractor by minimizing the classification loss computed by labeled source data and the estimated Wasserstein distance simultaneously. The learned representations can be domain invariant and target discriminative since the parameter $\theta_g$ receives the gradients from both the domain critic and the discriminator loss. 

\begin{algorithm}[t]
	\caption{Wasserstein Distance Guided Representation Learning Combining with Discriminator}\label{alg:framework}
	\begin{algorithmic}[1]
		\small
		\REQUIRE
		source data $X^s$; target data $X^t$; minibatch size $m$; critic training step $n$; coefficient $\gamma$, $\lambda$; learning rate for domain critic $\alpha_1$; learning rate for classification and feature learning $\alpha_2$
		\STATE
		Initialize feature extractor, domain critic, discriminator with random weights $\theta_g, \theta_w, \theta_c$
		\REPEAT
		\STATE	
		Sample minibatch $\{ x^s_i, y^s_i \}_{i=1}^m $, $\{x^t_i \}_{i=1}^m $ from $X^s$ and $X^t$
		\FOR{$t=1,...,n $}
		\STATE
		$h^s \leftarrow f_g(x^s) $, $h^t \leftarrow f_g(x^t) $
		\STATE
		Sample $h$ as the random points along straight lines between $h^s$ and $h^t$ pairs
		\STATE
		$\hat{h} \leftarrow \{h^s, h^t, h \}$
		\STATE
		$\theta_w \leftarrow \theta_w + \alpha_1 \nabla_{\theta_w}[\mathcal{L}_{wd}(x^s, x^t) - \gamma \mathcal{L}_{grad}(\hat{h})]$
		\ENDFOR
		\STATE
		$\theta_c \leftarrow \theta_c - \alpha_2 \nabla_{\theta_c}\mathcal{L}_c(x^s, y^s)$
		\STATE
		$\theta_g \leftarrow \theta_g - \alpha_2 \nabla_{\theta_g}[\mathcal{L}_c(x^s, y^s) + \mathcal{L}_{wd}(x^s, x^t)]$
		\UNTIL{$\theta_g, \theta_w, \theta_c$ converge}
	\end{algorithmic}
\end{algorithm}

\subsection{Theoretical Analysis}

In this section, we give some theoretical analysis about the advantages of using Wasserstein distance for domain adaptation. 

\subsubsection{Gradient Superiority}

In domain adaptation, to minimize the divergence between the data distributions $\mathbb{P}_{x^s}$ and $\mathbb{P}_{x^t}$, the symmetric feature-based methods learn a transformation function to map the data from the original space to a common latent space with a distance measure. There are two situations after the mapping:  
i). The two mapped feature distributions have supports that lie on low dimensional manifolds \cite{narayanan2010sample} in the latent space. In such situation, there will be a gradient vanishing problem if adopting the domain classifier to make data indistinguishable while Wasserstein distance could provide reliable gradients \cite{arjovsky2017wasserstein}.
ii). The feature representations may fill in the whole space since the feature mapping usually reduces dimensionality. However, if a data point lies in the regions where the probability of one distribution could be ignored compared with the other distribution, it makes no contributions to the gradients with traditional cross-entropy loss since the gradient computed by this data point is almost $0$. If we adopt Wasserstein distance as the distance measure, stable gradients can be provided wherever. So theoretically in either situation, WDGRL can perform better than previous adversarial adaptation methods \cite{ganin2016domain,tzeng2017adversarial}.

\subsubsection{Generalization Bound}

\cite{redko2016theoretical} proved that the target error can be bounded by the Wasserstein distance for empirical measures. However, the generalization bound exists when assuming the hypothesis class is a unit ball in RKHS and the transport cost function is RKHS distance. In this paper we prove the generalization bound in terms of the Kantorovich-Rubinstein dual formulation under a different assumption.

We first formalize some notations that will be used in the following statements. Let $\mathcal{X}$ be an instance set and $\{0, 1 \}$ be the label set for binary classification. We denote by $\mu_s$ the distribution of source instances on $\mathcal{X}$ and use $\mu_t$ for the target domain. We denote that two domains have the same labeling function $f:\mathcal{X} \rightarrow [0,1]$ which is always assumed to hold in domain adaptation problem. A hypothesis class $H$ is a set of predictor functions, $\forall h \in H, h:\mathcal{X} \rightarrow [0,1] $. The probability according to the distribution $\mu_s$ that a hypothesis $h$ disagrees with the labeling function $f$ (which can also be a hypothesis) is defined as $\epsilon_s(h,f) = \mathbb{E}_{x\in \mu_s}[|h(x)-f(x)|]$. We use the shorthand $\epsilon_s(h) = \epsilon_s(h,f)$ and $\epsilon_t(h)$ is defined the same. We now present the Lemma that introduces Wasserstein distance to relate the source and target errors.

\begin{lemma}
	\label{lemma: wd-lip}
	Let $ \mu_s,\mu_t \in \mathcal{P}(\mathcal{X}) $ be two probability measures. Assume the hypotheses $ h \in H$ are all $K$-Lipschitz continuous for some $K$. Then the following holds 
	\begin{equation}
	\epsilon_t(h, h') \leq \epsilon_s(h,h') + 2KW_1(\mu_s,\mu_t)
	\end{equation}
	for every hypothesis $h,h' \in H$.
\end{lemma}

\begin{proof}
	We first prove that for every $K$-Lipschitz continuous hypotheses $h,h' \in H$, $|h-h'|$ is  $2K$-Lipschitz continuous. Using the triangle inequality, we have
	{\small
		\begin{equation}
		\begin{aligned}
		|h(x) \! - \! h'(x)|& \! \leq \! |h(x) \! - \! h(y)| \! + \! |h(y) \! - \! h'(x)| \\
		& \! \leq \! |h(x) \! - \! h(y)| \! + \! |h(y) \!- \! h'(y)| \! + \! |h'(x) \! - \! h'(y)|
		\end{aligned}
		\end{equation}
	}
	and thus for every $x,y \in \mathcal{X}$,
	{\small
		\begin{equation}
		\begin{aligned}
		\frac{|h(x) \! - \! h'(x)| \! - \! |h(y) \! - \! h'(y)|}{\rho(x,y)}  &  \! \leq \! \frac{|h(x) \! - \! h(y)| \! + \! |h'(x) \! - \! h'(y)|}{\rho(x,y)} \\ & \leq 2K.
		\end{aligned}
		\end{equation}
	}
	Then for every hypothesis $h,h'$, we have
	{\small
		\begin{equation}
		\begin{aligned}
		\epsilon_t(h,h') \! - \! \epsilon_s(h,h') & \! = \! \mathbb{E}_{\mu_t}[|h(x) \! - \! h'(x)|] \! - \! \mathbb{E}_{\mu_s}[|h(x) \! - \! h'(x)|]  \\
		& \! \leq \! \sup_{\left \| f \right \|_L \leq 2K} \mathbb{E}_{\mu_t}[f(x)] \! - \! \mathbb{E}_{\mu_s}[f(x)] \\
		& \! = \! 2K W_1(\mu_s,\mu_t)
		\end{aligned}
		\end{equation}
	}
\end{proof}

\begin{theorem}
	\label{theo: wd-bound}
	Under the assumption of Lemma~\ref{lemma: wd-lip}, for every $h \in H$ the following holds
	\begin{equation}
	\epsilon_t(h) \leq \epsilon_s(h) + 2KW_1(\mu_s,\mu_t) + \lambda
	\end{equation}
	where $\lambda$ is the combined error of the ideal hypothesis $h^*$ that minimizes the combined error $\epsilon_s(h)+\epsilon_t(h)$.
\end{theorem}

\begin{proof}
	\begin{equation}
	\begin{aligned}
	\epsilon_t(h) & \leq \epsilon_t(h^*) + \epsilon_t(h^*, h) \\
	& = \epsilon_t(h^*) + \epsilon_s(h, h^*) + \epsilon_t(h^*,h) - \epsilon_s(h,h^*) \\
	& \leq \epsilon_t(h^*) + \epsilon_s(h, h^*) + 2KW_1(\mu_s, \mu_t)  \\
	& \leq \epsilon_t(h^*) + \epsilon_s(h) + \epsilon_s(h^*) + 2KW_1(\mu_s, \mu_t) \\
	& = \epsilon_s(h) + 2KW_1(\mu_s,\mu_t) + \lambda
	\end{aligned}
	\end{equation}
\end{proof}

Thus the generalization bound of applying Wasserstein distance between domain distributions has been proved, while the proof of using empirical measures on the source and target domain samples can be further proved according to Theorem 2.1 in \cite{bolley2007quantitative} as the same way in \cite{redko2016theoretical}.

The assumption made here is to specify the hypothesis class is $K$-Lipschitz continuous for some $K$. While it may seem too restrictive, in fact the hypotheses are always implemented by neural networks where the basic linear mapping functions and the activation functions such as sigmoid and relu are all Lipschitz continuous, so the assumption is not that strong and can be fulfilled. And the weights in neural networks are always regularized to avoid overfitting which means the constant $K$ will not be too large. Compared with the proof in \cite{redko2016theoretical} the assumptions are different and can be used for different cases.

\subsection{Application to Adaptation Frameworks}

WDGRL can be integrated into existing feature-based domain adaptation frameworks \cite{tzeng2014deep,long2015learning,zhuang2015supervised,long2016deep,bousmalis2016domain}. These frameworks are all symmetric feature-based and aim to learn domain invariant feature representations for adaptation using divergence measures such as MMD and DANN. We provide a promising alternative WDGRL to learn domain invariant representations, which can replace the MMD or DANN. We should point out that although WDGRL has gradient advantage over DANN, it takes more time to estimate the Wasserstein distance. 
Although we only apply WDGRL on one hidden layer, it can also be applied on multilayer structures as implemented in \cite{long2015learning}.

\section{Experiments}

In this section, we evaluate the efficacy of our approach on sentiment and image classification adaptation datasets. Compared with other domain invariant representation learning approaches, WDGRL achieves better performance on average. Furthermore, we visualize the feature representations learned by these approaches for an empirical analysis.

\subsection{Datasets}

\textbf{Amazon review benchmark dataset.} The Amazon review dataset\footnote{https://www.cs.jhu.edu/\textasciitilde mdredze/datasets/sentiment/} \cite{blitzer2007biographies} is one of the most widely used benchmarks for domain adaptation and sentiment analysis. It is collected from product reviews from Amazon.com and contains four types (domains), namely books (B), DVDs (D), electronics (E) and kitchen appliances (K). For each domain, there are 2,000 labeled reviews and approximately 4,000 unlabeled reviews (varying slightly across domains) and the classes are balanced. In our experiments, for easy computation, we follow \cite{chen2012marginalized} to use the 5,000  most frequent terms of unigrams and bigrams as the input and totally $A_4^2=12$ adaptation tasks are constructed.

\textbf{Office-Caltech object recognition dataset.} The Office-Caltech dataset\footnote{https://cs.stanford.edu/\textasciitilde jhoffman/domainadapt/} released by \cite{gong2012geodesic} is comprised of 10 common categories shared by the Office-31 and Caltech-256 datasets. In our experiments, we construct 12 tasks across 4 domains: Amazon (A), Webcam (W), DSLR (D) and Caltech (C), with 958, 295, 157 and 1,123 image samples respectively. In our experiments, Decaf features are used as the input. Decaf features \cite{donahue2014decaf} are the 4096-dimensional FC7-layer hidden activations extracted by the deep convolutional neural network AlexNet. 

\subsection{Compared Approaches}
We mainly compare our proposed approach with domain adversarial neural network (DANN) \cite{ganin2016domain}, maximum mean discrepancy metric (MMD) \cite{gretton2012kernel} and deep correlation alignment (CORAL) \cite{sun2016deep} since these approaches and our proposed WDGRL all aim at learning the domain invariant feature representations, which are crucial to reduce the domain discrepancy. Other domain adaptation frameworks \cite{bousmalis2016domain,tzeng2014deep,long2015learning,long2016deep,zhuang2015supervised} are not included in the comparison, because these frameworks focus on adaptation architecture design and all compared approaches can be easily integrated into these frameworks.

\textbf{S-only}: As an empirical lower bound, we train a model using the labeled source data only, and test it on the target test data directly. 

\textbf{MMD}: The MMD metric is a measurement of the divergence between two probability distributions from their samples by computing the distance of mean embeddings in RKHS.

\textbf{DANN}: DANN is an adversarial representation learning approach that a domain classifier aims at distinguishing the learned source/target features while the feature extractor tries to confuse the domain classifier. The minimax optimization is solved via a gradient reversal layer (GRL).

\textbf{CORAL}: Deep correlation alignment minimizes domain discrepancy by aligning the second-order statistics of the source and target distributions and can be applied to the layer activations in neural networks.

\subsection{Implementation Details}

We implement all our experiments\footnote{Experiment code: https://github.com/RockySJ/WDGRL.} using TensorFlow and the models are all trained with Adam optimizer. We follow the evaluation protocol in \cite{Long_2013_ICCV} and evaluate all compared approaches through grid search on the hyperparameter space, and report the best results of each approach. For each approach we use a batch size of 64 samples in total with 32 samples from each domain, and a fixed learning rate $10^{-4}$. All compared approaches are combined with a discriminator to learn both domain invariant and discriminative representations and to conduct the classification task. 

We use standard multi-layer perceptron (MLP) as the basic network architecture. MLP is sufficient to handle all the problems in our experiments. For Amazon review dataset the network is designed with one hidden layer of 500 nodes, relu activation function and softmax output function, while the network for Office-Caltech dataset has two hidden layers of 500 and 100 nodes. For each dataset the same network architecture is used for all compared approaches and these approaches are all applied on the last hidden layer.

For the MMD experiments we follow the suggestions of \cite{bousmalis2016domain} and use a linear combination of 19 RBF kernels with the standard deviation parameters ranging from \(10^{-6}\) to \(10^6\). As for DANN implementation, we add a gradient reversal layer (GRL) and then a domain classifier with one hidden layer of 100 nodes. And the CORAL approach computes a distance between the second-order statistics (covariances) of the source and target features and the distance is defined as the squared Frobenius norm. For each approach, the corresponding loss term is added to the classification loss with a coefficient for the trade-off. And the coefficients are tuned different to achieve the best results for each approach.

Our approach is easy to implement according to Algorithm~\ref{alg:framework}. In our experiments, the domain critic network is designed with a hidden layer of 100 nodes. The training steps $n$ is 5 which is chosen for fast computation and sufficient optimization guarantee for the domain critic, and the learning rate for the domain critic is $10^{-4}$. We penalize the gradients not only at source/target representations but also at the random points along the straight line between the source and target pairs and the coefficient $\gamma$ is set to 10 as suggested in \cite{gulrajani2017improved}. 

\subsection{Results and Discussion}

\begin{table}[t]
	\small
	\caption{Performance (accuracy \%) on Amazon review dataset.} \label{tab:amazon-result}
	\centering
	\begin{tabular}{cccccc}
		\hline 
		& S-only & MMD & DANN & CORAL & WDGRL \\
		\hline
		B \(\rightarrow \) D & 81.09 & 82.57& 82.07 &82.74& \textbf{83.05} \\
		B \(\rightarrow \) E & 75.23 & 80.95& 78.98 &82.93& \textbf{83.28} \\
		B \(\rightarrow \) K & 77.78 & 83.55& 82.76 &84.81& \textbf{85.45} \\
		\hline
		D \(\rightarrow \) B & 76.46 & 79.93& 79.35 &\textbf{80.81}& 80.72 \\
		D \(\rightarrow \) E & 76.24 & 82.59& 81.64 &83.49& \textbf{83.58} \\
		D \(\rightarrow \) K & 79.68 & 84.15& 83.41 &85.35& \textbf{86.24} \\
		\hline
		E \(\rightarrow \) B & 73.37 & 75.72& 75.95 &76.91& \textbf{77.22} \\
		E \(\rightarrow \) D & 73.79 & 77.69& 77.58 &78.08& \textbf{78.28} \\
		E \(\rightarrow \) K & 86.64 & 87.37& 86.63 &87.87& \textbf{88.16} \\
		\hline
		K \(\rightarrow \) B & 72.12 & 75.83& 75.81 &76.95& \textbf{77.16} \\
		K \(\rightarrow \) D & 75.79 & 78.05& 78.53 &79.11& \textbf{79.89} \\
		K \(\rightarrow \) E & 85.92 & 86.27& 86.11 &\textbf{86.83}& 86.29 \\
		\hline
		AVG & 77.84 & 81.22 & 80.74 & 82.16 & \textbf{82.43} \\
		\hline
	\end{tabular}
\end{table}

\textbf{Amazon review benchmark dataset.} The challenge of cross domain sentiment analysis lies in the distribution shift as different words are used in different domains.  Table~\ref{tab:amazon-result} shows the detailed comparison results of these approaches in 12 transfer tasks. As we can see, our proposed WDGRL outperforms all other compared approaches in 10 out of 12 domain adaptation tasks, and it achieves the second highest scores in the remaining 2 tasks. We find that as adversarial adaptation approaches, WDGRL outperforms DANN, which is consistent with our theoretical analysis that WDGRL has more reliable gradients. MMD and CORAL are both non-parametric and have lower computational cost than WDGRL, while their classification performances are also lower than WDGRL.

\textbf{Office-Caltech object recognition dataset.} Table~\ref{tab:office-result} shows the results of our experiments on Office-Caltech dataset. We observe that our approach achieves better performance than other compared approaches on most tasks. Office-Caltech dataset is small since there are only hundreds of images in one domain and it is a 10-class classification problem. Thus we can draw a conclusion that the empirical Wasserstein distance can also be applied to small-scale datasets adaptation effectively. We note that CORAL performs better than MMD in Amazon review dataset while it performs worse than MMD in Office-Caltech dataset. A possible reason is that the reasonable covariance alignment approach requires large samples. On the other hand, we can see that these different approaches have different performances on different adaptation tasks.

\begin{table}[t]
	\small
	\caption{Performance (accuracy \%) on Office-Caltech dataset with Decaf features.}
	\centering\label{tab:office-result}
	\begin{tabular}{cccccc}
		\hline 
		& S-only & MMD & DANN & CORAL & WDGRL \\
		\hline
		A \(\rightarrow \) C & 84.55 & \textbf{88.62} & 87.80 & 86.18 & 86.99 \\
		A \(\rightarrow \) D & 81.05 & 90.53 & 82.46 & 91.23 & \textbf{93.68} \\
		A \(\rightarrow \) W & 75.59 & \textbf{91.58} & 77.81 & 90.53 & 89.47 \\
		\hline
		W \(\rightarrow \) A & 79.82 & 92.22 & 82.98 & 88.39 & \textbf{93.67 }\\
		W \(\rightarrow \) D & 98.25 & \textbf{100} & \textbf{100} & \textbf{100} & \textbf{100} \\
		W \(\rightarrow \) C & 79.67 & 88.62 & 81.30 & 88.62 & \textbf{89.43} \\
		\hline
		D \(\rightarrow \) A & 84.56 & 90.11 & 84.70 & 85.75 & \textbf{91.69} \\
		D \(\rightarrow \) W & 96.84 & \textbf{98.95} &\textbf{ 98.95} & 97.89 & 97.89 \\
		D \(\rightarrow \) C & 80.49 & 87.80 & 82.11 & 85.37 & \textbf{90.24} \\
		\hline
		C \(\rightarrow \) A & 92.35 & 93.14 & 93.27 & 93.01 & \textbf{93.54} \\
		C \(\rightarrow \) W & 84.21 & 91.58 & 89.47 & \textbf{92.63} & 91.58 \\
		C \(\rightarrow \) D & 87.72 & 91.23 & 91.23 & 89.47 & \textbf{94.74} \\
		\hline
		AVG & 85.44 & 92.03 & 87.67 & 90.76 & \textbf{92.74} \\
		\hline
	\end{tabular}
\end{table}

\subsection{Feature Visualization}

We randomly choose the D\(\rightarrow \)E domain adaptation task of Amazon review dataset and plot in Figure~\ref{fig:t-sne} the t-SNE visualization following \cite{donahue2014decaf,long2016deep} to visualize the learned feature representations. In these figures, red and blue points represent positive and negative samples of the source domain, purple and green points represent positive and negative samples of the target domain. A transferable feature mapping should cluster red (blue) and purple (green) points together, and meanwhile classification can be easily conducted between purple and green points. We can see that almost all approaches learn discriminative and domain invariant feature representations to some extent. And representations learned by WDGRL are more transferable since the classes between the source and target domains align better and the region where purple and green points mix together is smaller.

\begin{figure}[tbp]
	\subfigure[t-SNE of DANN features]{
		\includegraphics[width=0.22\textwidth]{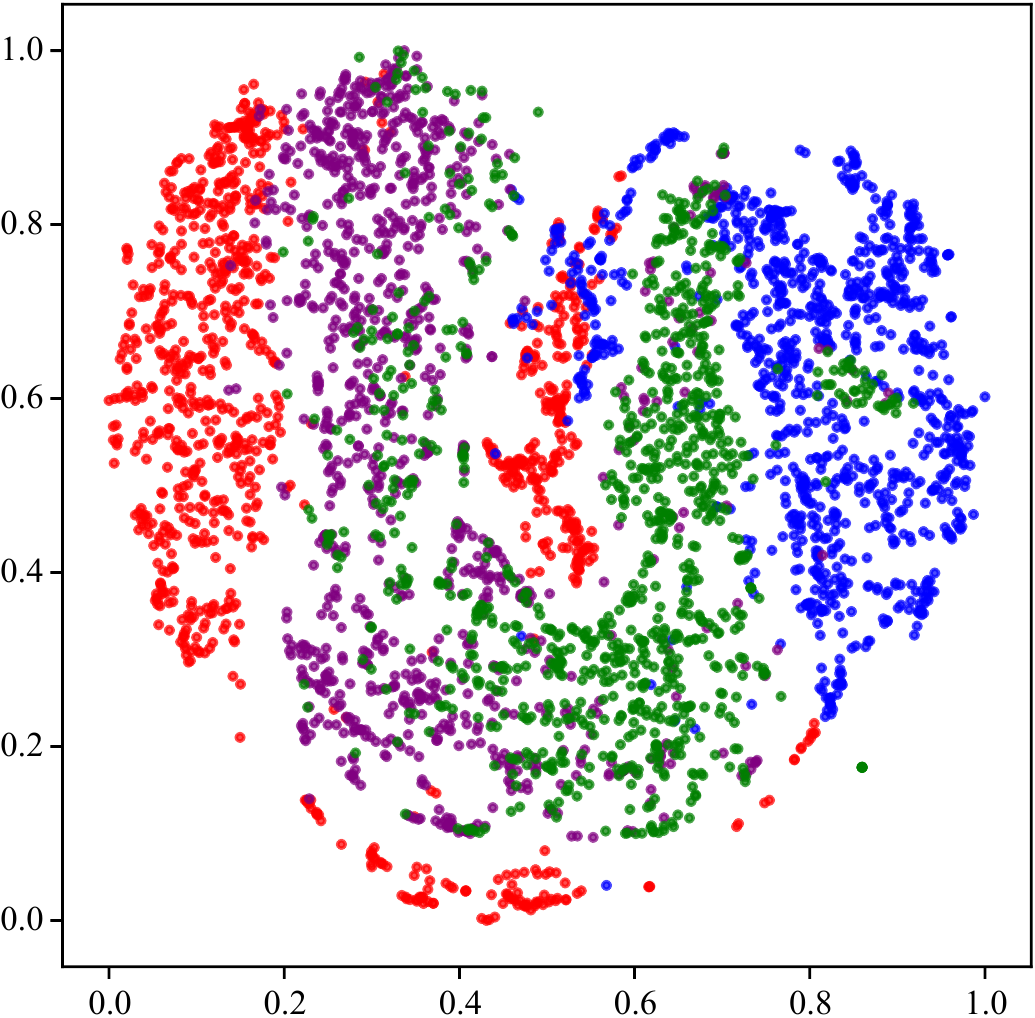}
	}
	\subfigure[t-SNE of MMD features]{
		\includegraphics[width=0.22\textwidth]{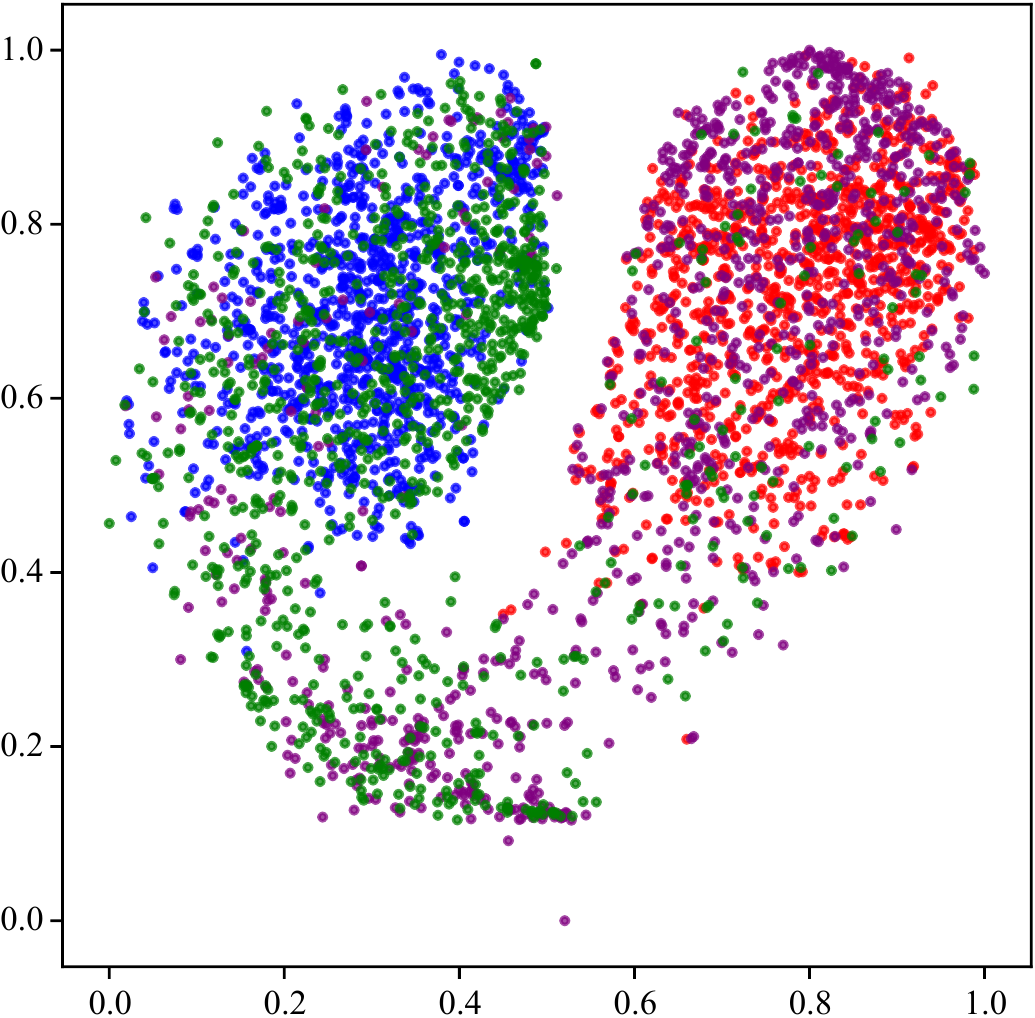}
	}
	
	\subfigure[t-SNE of CORAL features]{
		\includegraphics[width=0.22\textwidth]{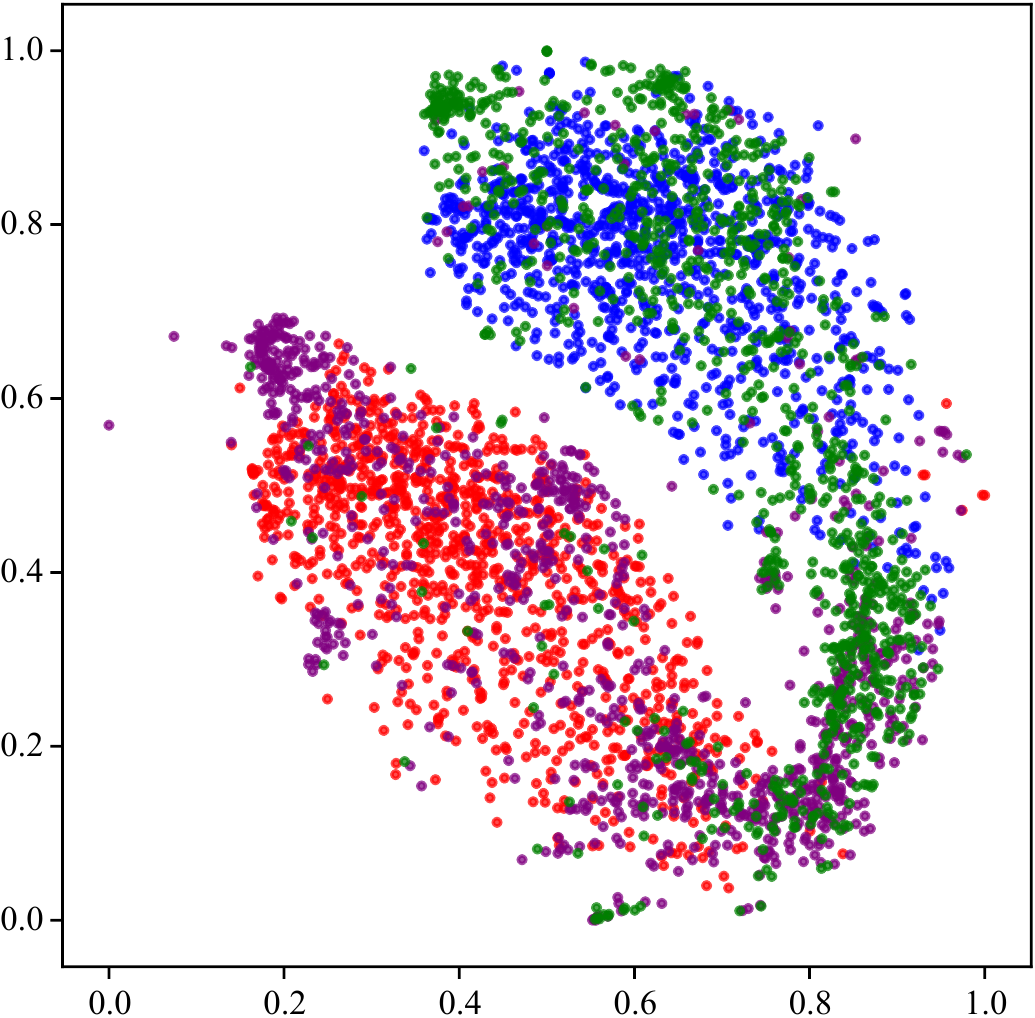}
	}
	\subfigure[t-SNE of WDGRL features]{
		\includegraphics[width=0.22\textwidth]{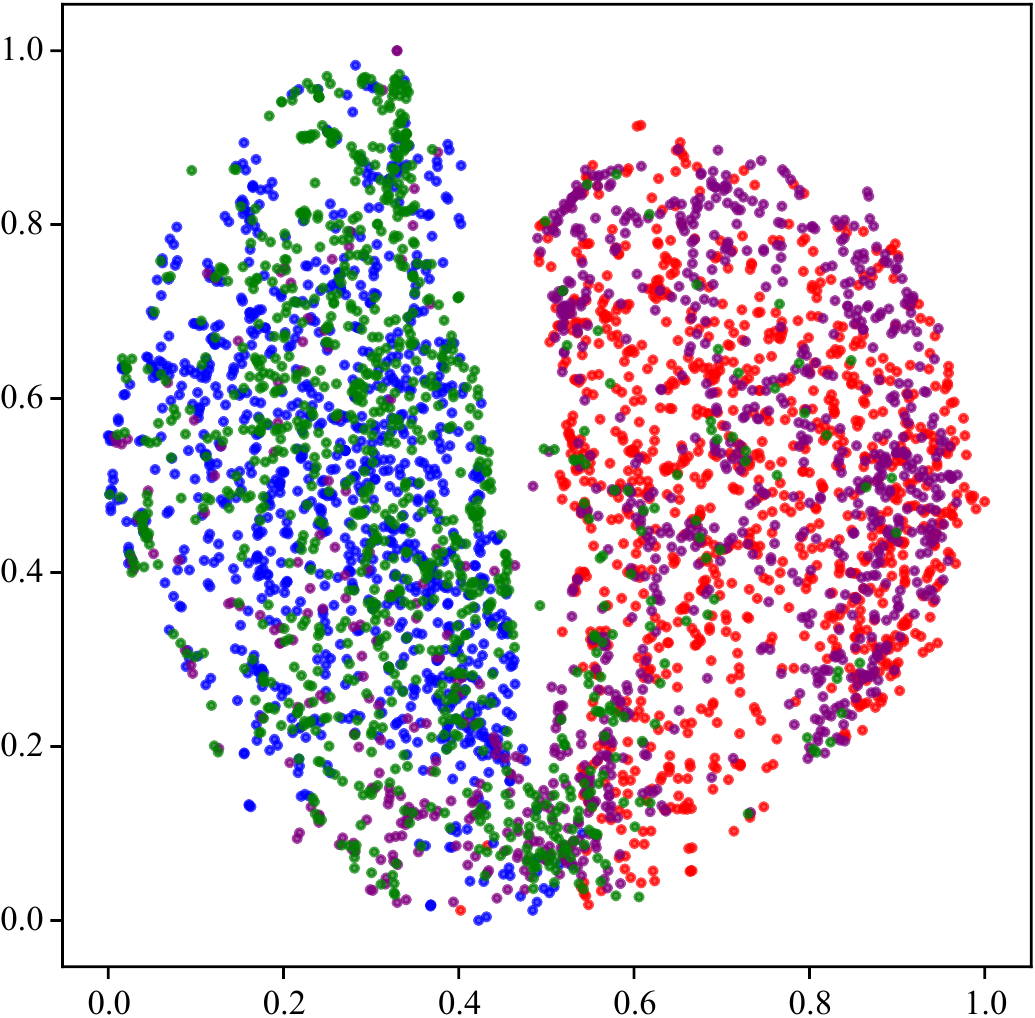}
	}
	
	\caption{Feature visualization of the D\(\rightarrow \)E task in Amazon review dataset.} 
	\label{fig:t-sne}
\end{figure}

\section{Conclusions}
In this paper, we propose a new adversarial approach WDGRL to learn domain invariant feature representations for domain adaptation. WDGRL can effectively reduce the domain discrepancy taking advantage of the gradient property of Wasserstein distance and the transferability is guaranteed by the generalization bound.
Our proposed approach could be further integrated into other domain adaptation frameworks \cite{bousmalis2016domain,tzeng2014deep,long2015learning,long2016deep,zhuang2015supervised} to attain better transferability. Empirical results on sentiment and image classification domain adaptation datasets demonstrate that WDGRL outperforms the state-of-the-art domain invariant feature learning approaches. From feature visualization, one can easily observe that WDGRL yields domain invariant yet target-discriminative feature representations.
In future work, we will investigate more sophisticated architectures for tasks on image data as well as integrate WDGRL into existing adaptation frameworks.

\section{Acknowledgement}

This work is financially supported by NSFC (61702327) and Shanghai Sailing Program (17YF1428200).

\bibliography{wd-tl}
\bibliographystyle{aaai}

\onecolumn

\section{Appendix}

\subsection{Gradient Superiority}

Here we would like to prove the gradient priority of Wasserstein distance over cross-entropy in the situation where the mapped feature distributions fill in the whole feature space.
For simplicity, we take two normal distributions as an example and the conclusion still holds in the high-dimensional space.
Fig~\ref{fig:gaussian} shows the two normal distributions and the whole space is divided into 3 regions where the probability of source data lying in region A is high while that of target data is extremely low. The situation is just opposite in region C and in region B two distributions differ a little.

\begin{figure}[htpb]
	{
			\centering
			\includegraphics[width=0.4\textwidth]{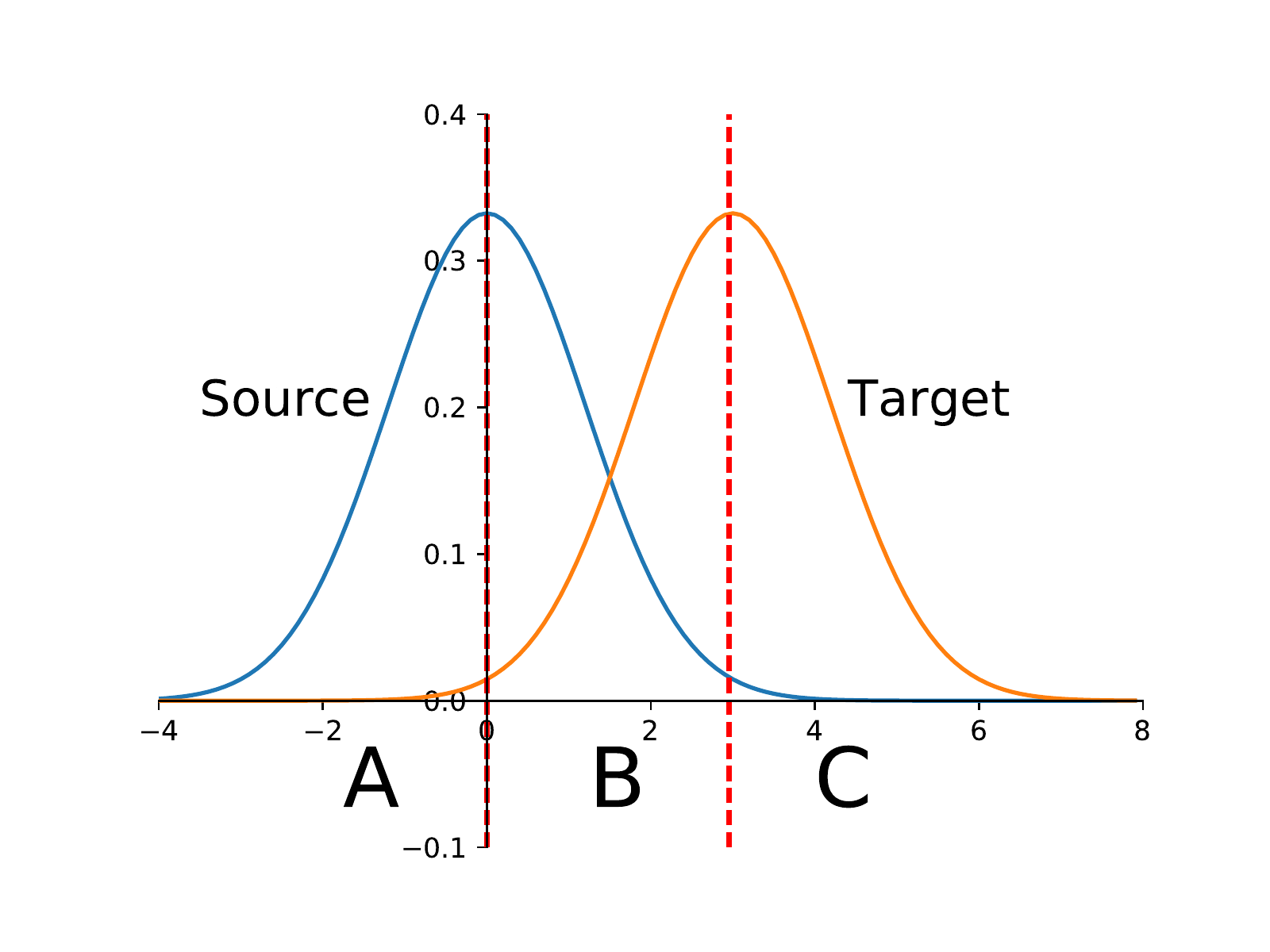}
		
	}
	\caption{Gaussian Example}
	\label{fig:gaussian}
\end{figure}

We use the same notation here as above. We assume that source data are labeled 1 while target data are labeled 0 and a domain classifier is used to help learn the domain invariant representations. So given one instance $(x, y)$ from either domain, the feature extractor minimizes the following objective which could be viewed as the negative of cross-entropy between the domain label $y$ and its corresponding prediction $\sigma(f_d(f_g(x)))$
\begin{equation}
\mathcal{L}_D(x, y) = y \log \sigma  (f_d(f_g(x))) + (1-y) \log (1-\sigma(f_d(f_g(x))))
\end{equation}
where $\sigma$ is the sigmoid function and $f_d$ is the logit computed by the domain classifier network. Then the gradient of $\mathcal{L}_D $ with respect to $\theta_g $ can be computed according to the chain rule, i.e. 
$
\frac{\partial \mathcal{L}_D}{\partial \theta_g} = \frac{\partial \mathcal{L}_D} {\partial f_d} \frac{\partial f_d} {\partial f_g} \frac{\partial f_g}{\partial \theta_g}
$. The first term can be directly computed 
\begin{equation}
\frac{\partial \mathcal{L}_D} {\partial f_d} = y - \sigma(f_d(f_g))
\end{equation}
As we know, the optimal domain classifier is $\sigma(f_d^*(h)) = \frac{p(h)}{p(h)+q(h)}$ where $h=f_g(x)$ and $p(h)$ represents the source feature distribution and $q(h)$ represents the target feature distribution. So if one source instance lies in region A, it provides gradient of almost 0. The same result holds for target samples lying in region C. So these points make no contribution to the gradient and thus the divergence between feature distributions couldn't be reduced effectively.

Now we consider Wasserstein distance as the loss function
\begin{equation}
\mathcal{L}_W = \mathbb{E}_{x \sim \mathbb{P}_{x^s}} [f_w(f_g(x))] - \mathbb{E}_{x \sim \mathbb{P}_{x^t}  }[f_w(f_g(x))].
\end{equation}
The gradient of $\mathcal{L}_W $ with respect to $\theta_g $ can be computed according to the chain rule, i.e. 
$
\frac{\partial \mathcal{L}_W}{\partial \theta_g} = \frac{\partial \mathcal{L}_W} {\partial f_w} \frac{\partial f_w} {\partial f_g} \frac{\partial f_g}{\partial \theta_g}
$.
So for source domain data $x \sim \mathbb{P}_{x^s}$, $\frac{\partial \mathcal{L}_W}{\partial f_w} = 1$; while for target domain data $x \sim \mathbb{P}_{x^t}$, $\frac{\partial \mathcal{L}_W}{\partial f_w} = -1$. Therefore Wasserstein distance can always provide stable gradients wherever data is.

\subsection{Generalization Bound}	

We now continue from the Theorem 1 in the paper to prove that target error can be bounded by the Wasserstein distance for empirical measures on the source and target samples. we first present a statement showing the convergence of the empirical measure to the true Wasserstein distance.

\begin{theorem}
\label{theo: wd-emprical}
(\cite{bolley2007quantitative}, Theorem 2.1; \cite{redko2016theoretical}, Theorem 1) Let $\mu$ be a probability measure in $\mathbb{R}^d$ 
satisfying $T_1(\lambda)$ inequality. 
Let $\hat{\mu} = \frac{1}{N} \sum_{i=1}^{N} \delta_{x_i}$ be its associated empirical defined on a sample of independent variables $ \{x_i\} _{i=1}^N $ drawn from $\mu$. Then for any $d'>d$ and $\lambda'< \lambda$ there exists some constant $N_0$ depending on $d'$ and some square exponential moment of $\mu $ such that for any $\epsilon>0$ and $N \geq N_0\text{max}(\varepsilon^{-(d+2)},1)$
\begin{equation}
\mathbb{P}[W_1(\mu,\hat{\mu})>\varepsilon] \leq \text{exp} \big( - \frac{\lambda'}{2} N \varepsilon^2  \big)
\end{equation} 
where $d', \lambda'$ can be calculated explicitly.
\end{theorem}

Now we can follow the Theorem~\ref{theo: wd-bound} and Theorem~\ref{theo: wd-emprical} to prove that target error can be bounded by the Wasserstein distance for empirical measures on the source and target samples as the process of the proof of the Theorem 3. in \cite{redko2016theoretical}.

\begin{theorem}
Under the assumption of Lemma 1, let two probability measures satisfy $T_1(\lambda)$ inequality, $X_s$ and $X_t$ be two samples of size $N_s$ and $N_t$ drawn i.i.d from $\mu_s$ and $\mu_t$ resepectively. Let $\hat{\mu_s} = \frac{1}{N_s}\sum_{i=1}^{N_s} \delta _{x_i^s}$ and $\hat{\mu_t} = \frac{1}{N_t}\sum_{i=1}^{N_t} \delta _{x_i^t}$ be the associated empirical measures. Then for any $d' > d$ and $ \lambda' < \lambda $ there exists some constant $N_0$ depending on $d'$ such that for any $\delta >0 $ and $\text{min}(Ns,Nt) \geq N_0\text{max}(\delta^{-(d'+2)},1)$ with probability at least $1-\delta$ for all $h$ the followingt holds:
\begin{equation}
\epsilon_t(h) \leq \epsilon_s(h) + 2K W_1(\hat{\mu_s},\hat{\mu_t}) + \lambda + 2K \sqrt{2log \bigg(\frac{1}{\delta} \bigg) / \lambda'} \bigg( \sqrt{\frac{1}{N_s}} + \sqrt{\frac{1}{N_t}} \bigg)
\end{equation}
where $\lambda$ is the combined error of the ideal hypothesis $ h^* $ that minimizes the combined error of $\epsilon_s(h)+\epsilon_t(h)$.
\end{theorem}

\begin{proof}
\begin{equation}
\begin{aligned}
\epsilon_t(h) 
& \leq \epsilon_s(h) + 2KW_1(\mu_s,\mu_t) + \lambda \\
& \leq \epsilon_s(h) + 2KW_1(\mu_s,\hat{\mu_s}) + 2KW_1(\hat{\mu_s}, \mu_t) + \lambda \\
& \leq \epsilon_s(h) + 2K \sqrt{2log \bigg(\frac{1}{\delta} \bigg) / N_s\lambda'} + 2KW_1(\hat{\mu_s}, \hat{\mu_t}) + 2KW_1(\hat{\mu_t,\mu_t}) + \lambda \\
& \leq \epsilon_s(h) + 2KW_1(\hat{\mu_s}, \hat{\mu_t}) + \lambda + 2K \sqrt{2log \bigg(\frac{1}{\delta} \bigg) / \lambda'} \bigg( \sqrt{\frac{1}{N_s}} + \sqrt{\frac{1}{N_t}} \bigg)
\end{aligned}
\end{equation}
\end{proof}

\subsection{More Experiment Results}

\textbf{Synthetic data.} We generate a synthetic dataset to show the superior gradient advantage of WDGRL over DANN. In the paper, we claim that when two representation distributions are distant or have regions they differ a lot, DANN will have gradient vanishing problem while WDGRL still provides the stable gradient. It is a little difficult to fully realize such situations, so we design a rather restrictive experiment. However, this toy experiment does verify DANN may fail in some situations while WDGRL can work. We visualize the data input in Figure~\ref{fig: synthetic-data} with 2000 samples for each domain. And from Figure~\ref{fig: dann_clssifier} we find that if we adopt DANN the domain classifier can distinguish two domain data well and the DANN loss decreases to nearly 0 as the training process continues. In such situation, the domain classifier can provide poor gradient. As shown in ~\ref{fig: target_acc}, our WDGRL approach can effectively classify the target data while DANN fails.

\begin{figure}[htbp]
	\subfigure[input visualization]{
		\centering
		\includegraphics[width=0.31\textwidth]{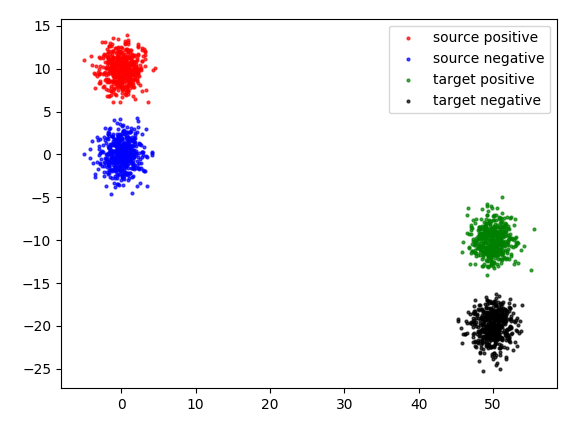}
		\label{fig: synthetic-data}
	}
	\subfigure[DANN loss and accuracy]{
		\centering
		\includegraphics[width=0.31\textwidth]{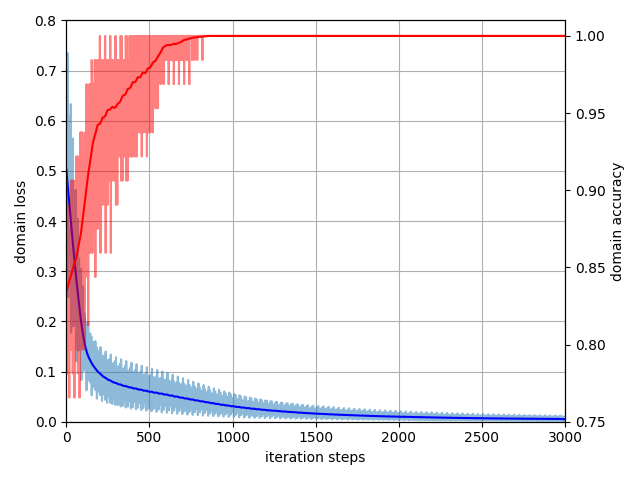}
		\label{fig: dann_clssifier}
	}
	\subfigure[Performance on target domain]{
		\centering
		\includegraphics[width=0.31\textwidth]{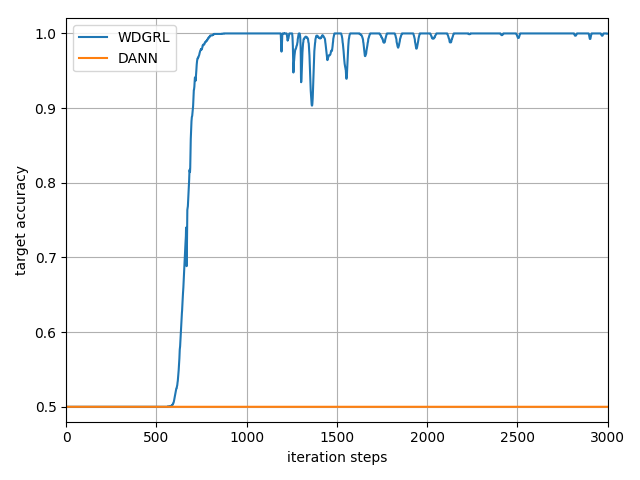}
		\label{fig: target_acc}
	}
	\caption{Synthetic experiment.}
\end{figure}

\textbf{Office-Caltech dataset with SURF features.} Table~\ref{tab:office-result-surf} shows the result of our experiments on Office-Caltech dataset with SURF features.

\begin{table}[ht]
	\small
	\caption{Performance (accuracy \%) on Office-Caltech dataset with Decaf features}
	\centering\label{tab:office-result-surf}
	\begin{tabular}{cccccc}
		\hline 
		& S-only & MMD & DANN & D-CORAL & WDGRL \\
		\hline
		A \(\rightarrow \) C & 43.19 & 44.08 & 44.97 & 44.97 & \textbf{45.86} \\
		A \(\rightarrow \) D & 35.03 & 41.40 & 41.40 & 40.13 & \textbf{44.59} \\
		A \(\rightarrow \) W & 35.23 & 37.29 & 38.64 & 38.31 & \textbf{40.68} \\
		\hline
		W \(\rightarrow \) A & 30.06 & 34.13 & 34.13 & \textbf{34.86} & 32.15 \\
		W \(\rightarrow \) D & 80.25 &\textbf{ 84.71} & 82.80 & 84.08 & 81.53 \\
		W \(\rightarrow \) C & 30.19 & 30.72 & 32.68 & \textbf{33.30} & 31.08 \\
		\hline
		C \(\rightarrow \) W & 36.95 & 40.34 & \textbf{43.39} & 40.00 & 42.37 \\
		C \(\rightarrow \) A & 52.92 & 54.80 & 54.91 & 53.44 & \textbf{55.22} \\
		C \(\rightarrow \) D & 45.86 & 47.13 & 47.77 & 47.13 & \textbf{48.41} \\
		\hline
		D \(\rightarrow \) W & 69.50 & 73.56 & 74.24 & 73.90 & \textbf{76.95} \\
		D \(\rightarrow \) A & 31.21 & 32.46 & 31.63 & 31.52 & \textbf{35.60} \\
		D \(\rightarrow \) C & 30.37 & 30.72 & 32.24 & 31.52 & \textbf{32.59} \\
		\hline
		AVG & 43.4 & 45.95 & 46.57 & 46.10 & \textbf{47.25} \\
		\hline
	\end{tabular}
\end{table}

\textbf{Email spam filtering dataset.} The email spam filtering dataset \footnote{http://www.ecmlpkdd2006.org/challenge.html} released by ECML/PKDD 2006 discovery challenge contains 4 separate user inboxes. From public inbox (source domain) 4,000 labeled training samples were collected, among which half samples are spam emails and the other half non-spam ones. The test samples were collected from 3 private inboxes (target domains), each of which consists of 2,500 samples. In our experiments, 3 cross-domain tasks are constructed from the public inbox to the private inboxes. We choose the 5,067 most frequent terms as features and 4 test samples were deleted as a result of not containing any of these terms. Experimenting on the 3 tasks by transferring from public to private groups of private inboxes \(u1 \sim u3 \), we found our method does achieve better performance than MMD, DANN and D-CORAL, which is demonstrated in Table~\ref{tab:email-result}. We can see from this result that all these approaches can reach the goal of learning the transferable features for they all outperform the source only baseline at least \(9 \%\). Among them, MMD and DANN achieve almost the same performance while WDGRL further boosts the performance by a rate of \(2.90 \%\). 
\begin{table}[ht]
	\small
	\caption{Performance (Accuracy \%) on email spam dataset}
	\centering
	\label{tab:email-result}
	\begin{tabular}{cccccc}
		\hline 
		& S only & MMD & DANN & D-CORAL & WDGRL\\
		\hline
		P \(\rightarrow u1\)  & 69.63 & 80.95 & 83.27 &79.71& \textbf{85.67}\\
		P \(\rightarrow u2\)  & 76.01 & 85.98 & 85.74 &83.83& \textbf{88.26}\\
		P \(\rightarrow u3\)  & 81.24 & 94.08 & 91.92 &89.80& \textbf{95.76}\\
		\hline
		AVG                        & 75.63 & 87.00 & 86.98 &84.45& \textbf{89.90}\\
		\hline
		
	\end{tabular}
\end{table}

\textbf{Newsgroup classification dataset.} The 20 newsgroups dataset \footnote{http://qwone.com/\textasciitilde jason/20Newsgroups/} is a collection of 18,774 newsgroup documents across 6 top categories and 20 subcategories in a hierarchical structure. In our experiments, we adopt a similar setting as \cite{duan2012domain}. The task is to classify top categories and the four largest top categories (comp, rec, sci, talk) are chosen for evaluation. Specifically, for each top category, the largest subcategory is selected as the source domain while the second largest subcategory is chosen as the target domain. Moreover, the largest category comp is considered as the positive class and one of the three other categories as the negative class. 

The distribution shift across newsgroups is caused by category specific words. Notice the construction of our domain adaptation tasks which aim to classify the top categories while the adaptation exists between the subcategories. It makes sense that there exist more differences among top categories than those among subcategories which implies that classification is not that sensitive to the subcategories and thus enables the ease of domain adaptation. Table~\ref{tab:news-reuslt} gives the information of performance on the 20newsgroup dataset from which we can find that the comparison methods are almost neck and neck, which is consistent with our previous observation.

\begin{table}[ht]
	\small
	\caption{Performance (Accuracy \%) on 20 newsgroup dataset}
	\centering\label{tab:news-reuslt}
	\begin{tabular}{cccccc}
		\hline 
		& S only & MMD & DANN & D-CORAL &WDGRL\\
		\hline
		C vs. R   & 81.62 & 97.85 & 98.10 & 97.57 &\textbf{98.35}\\
		C vs. S   & 74.01 & 87.52 & 90.57 & 84.20 &\textbf{91.33}\\
		C vs. T   & 94.44 & 96.96 & \textbf{97.75} & 97.22 &97.62\\
		\hline
		AVG            & 83.36 & 94.11 & 95.47 & 93.00 & \textbf{95.77}\\
		\hline
		
	\end{tabular}
\end{table}

\end{document}